\documentclass[11pt]{article}
\usepackage{fullpage, url, graphicx, subfigure}
\usepackage{amsmath, amsthm, amssymb}
\usepackage{fancyhdr}
\usepackage{mathrsfs}
\usepackage{epsfig}
\usepackage{graphics}
\usepackage{epsf}
\usepackage{amsmath, amsthm, amssymb, multirow, paralist, mathtools}
\usepackage{fullpage}
\usepackage{url}
\usepackage{algorithm,algorithmic}

\title{Accelerating Minibatch Stochastic Gradient Descent using Stratified Sampling}

\date{}

\def \g  {\mathbf{g}}

\def \itj {i_t^j}

\def \u {\mathbf{u}}

\def \v {\mathbf{v}}
\def \w  {\mathbf{w}}

\def \x {\mathbf{x}}

\def \B {\mathcal{B}}
\def \C {\mathcal{C}}

\def \E {\mathbb{E}}

\def \R {\mathbb{R}}

\def \V {\mathbb{V}}

\def \bq  {\begin{eqnarray}}
\def \eq  {\end{eqnarray}}
\def \bqs {\begin{eqnarray*}}
\def \eqs {\end{eqnarray*}}

\newtheorem{property}{Property}
\newtheorem{definition}{Definition}
\newtheorem{lemma}{Lemma}
\newtheorem{thm}{Theorem}

\begin{document}
\author{Peilin Zhao \\
       Department of Statistics\\
       Rutgers University\\
       Piscataway, NJ, 08854, USA\\
       \emph{peilin.zhao@rutgers.edu}
       \and
       Tong Zhang  \\
       Baidu Inc. \& Rutgers University\\
       Piscataway, NJ, 08854, USA\\
       \emph{tzhang@stat.rutgers.edu}
}
\maketitle

\begin{abstract}
Stochastic Gradient Descent (SGD) is a popular optimization method which has been applied to many important machine learning tasks such as Support Vector Machines and Deep Neural Networks. In order to parallelize SGD, minibatch training is often employed.
The standard approach is to uniformly sample a minibatch at each step, which often leads to high variance.
In this paper we propose  a stratified sampling strategy, which divides the whole dataset into clusters with low within-cluster  variance; we then take examples from these clusters using a stratified sampling technique. It is shown that the convergence rate can be significantly improved by the algorithm. Encouraging experimental results confirm the effectiveness of the proposed method.
\end{abstract}

\section{Introduction}
Stochastic Gradient Descent has been extensively studied in the machine learning community~\cite{DBLP:conf/icml/Zhang04,duchi2009efficient,DBLP:journals/mp/Shalev-ShwartzSSC11,rakhlin2011making,DBLP:conf/nips/MahdaviYJZY12,shamir2013stochastic,johnson2013accelerating}. At every step, a typical stochastic gradient descent method will sample one training example uniformly at random from the training data, and then update the model parameter using the sampled example.
In order to parallelize SGD, the standard approach is to employ minibatch training, which samples multiple examples uniformly at each step.
The uniformly sampled minibatch stochastic gradient is an unbiased estimation of the true gradient~\cite{DBLP:conf/icml/Zhang04,rakhlin2011making,shamir2013stochastic,duchi2009efficient},
but the resulting estimator may have relatively high variance.
This will negatively affect the convergence rate of the underlying
optimization procedure. Instead of using uniform sampling of the training data, we propose to divide the whole dataset into clusters and perform stratified sampling that minimizes an upperbound of the variance.
We show that the proposed algorithm can significantly reduce the stochastic variance, which will then improve convergence.

The key idea in the proposed approach is to perform stratified sampling  and construct the corresponding unbiased stochastic gradient estimators that minimize an upperbound of the stochastic variance. To this end, we analyze the relationship between the variance of stochastic gradient estimator and the sampling strategy. We show that to minimize the variance, the optimal sampling strategy should roughly minimize a sum  of all the weighted standard deviations of gradients corresponding to the subdatasets. Our theoretical analysis shows that under certain conditions, the proposed sampling method can significantly improve the convergence rate. This result is empirically verified  by experiments.

The rest of this paper is organized as follows. Section~\ref{sec:related} reviews the related work. In section~\ref{sec:algorithm}, we study minibatch stochastic gradient descent with stratified sampling. The empirical evaluations are presented in Section~\ref{sec:experiment}.
Section~\ref{sec:conclusion} concludes the paper.

\section{Related Work}
\label{sec:related}
Stochastic Gradient Descent has been extensively studied in the traditional stochastic approximation literature~\cite{kushner2003stochastic}; however the results are often asymptotic. In recent years, finite sample convergence rate of SGD for solving linear prediction problems have been
studied by a number of authors~\cite{DBLP:conf/icml/Zhang04,DBLP:conf/icml/Shalev-ShwartzSS07}.
In general SGD can achieve a convergence rate of $O(1/\sqrt{T})$ for convex loss functions, and a convergence rate of $O(\log T/T)$ for strongly convex loss functions, where $T$ is the number of iterations of the algorithm.
More recently, researchers have improved the previous bound for
the strongly convex loss function case to $O(1/T)$ by using
$\alpha$-Suffix Averaging ~\cite{rakhlin2011making}, which means that
the algorithm will return the average of the last $\alpha$ fraction of
the previously obtained sequence of predictors. A similar result can be obtained via a polynomial decay averaging strategy~\cite{shamir2013stochastic}.

Although SGD has been extensively studied, most of the existing work only considered the uniform sampling scheme during the entire learning process, which will result in an unbiased estimator with high variance.
To explicitly reduce the variance, some new stochastic gradient algorithms have been proposed~\cite{johnson2013accelerating,DBLP:conf/nips/WangCSX13,DBLP:journals/corr/ZhaoZ14}.
In~\cite{johnson2013accelerating}, the authors constructed an unbiased stochastic estimator of the full gradient called SVRG, with the property that the resulting variance will approach zero asymptotically in the finite training example case. In~\cite{DBLP:conf/nips/WangCSX13}, the authors developed a variance reduction approach with control variates  formed using low-order moments. However they still employ uniform sampling during the training process. 

Instead of using uniform sampling, stochastic gradient descent with importance sampling was studied in~\cite{DBLP:journals/corr/ZhaoZ14}, where a nonuniform sampling distribution is constructed to reduce the variance of the stochastic gradient estimator.
This paper considers a different variance reduction method using
stratified sampling for minibatch SGD training. This idea is
complementary to previously proposed variance reduction methods  such as SVRG in \cite{johnson2013accelerating} and importance sampling in \cite{DBLP:journals/corr/ZhaoZ14}. In fact, these methods can be combined.

\section{Minibatch SGD with Stratified Sampling}
\label{sec:algorithm}
\subsection{Preliminaries}
\vspace{-0.1in}
\label{sec:prelim}
We briefly introduce some key definitions and a property of convex functions that are useful throughout the paper (for details, please refer to~\cite{borwein2006convex} ).
\begin{definition}
A function $\phi: \R^d\rightarrow\R$  is $H$-strongly convex, if for all $\u,\v \in \R^d$, we have
\begin{eqnarray*}
\phi(\u) \ge \phi(\v) + \nabla\phi(\v)^\top(\u-\v) + \frac{H}{2}\|\u-\v\|^2 ,
\end{eqnarray*}
where $\|\cdot\|$ is a norm.
\end{definition}
\begin{definition}
A function $\phi: \R^d\rightarrow\R$ is $L$-Lipschitz, if for all $\u,\v \in\R^d$, we have
\begin{eqnarray*}
|\phi(\u)-\phi(\v)|\le L \|\u-\v\|.
\end{eqnarray*}
\end{definition}
\begin{definition}
A function $\phi: \R^d\rightarrow\R$ is $(1/\gamma)$-smooth if it is differentiable and its gradient is $(1/\gamma)$-Lipschitz, or, equivalently for all $\u,\v\in\R^d$, we have
\begin{eqnarray*}
\phi(\u) \le \phi(\v) + \nabla\phi(\v)^\top(\u-\v) + \frac{1}{2\gamma}\|\u-\v\|^2 .
\end{eqnarray*}
\end{definition}
\begin{property}
If a function $\phi:\R^d\rightarrow \R$ is convex ($0$-strongly convex) and $(1/\gamma)$-smooth, then for all $\u,\v\in\R^d$, we have
\bqs
\langle\nabla \phi(\u)-\nabla \phi(\v), \u-\v\rangle \ge \gamma \|\nabla \phi(\u)-\nabla \phi(\v)\|^2,
\eqs
which is known as co-coercivity of $\nabla\phi$ with parameter $\gamma$.
\end{property}

Throughout this paper, we will denote $\|\cdot\|_2$ as $\|\cdot\|$ for simplicity.

\subsection{Problem Setting}
In this paper, we will focus on the standard multiclass classification task, for which a set of examples is given as $\{(\x_1,y_1),(\x_2,y_2),\ldots,(\x_n,y_n)\}$, where each $\x_i\in \R^d$ is a $d$-dimensional instance and $y_i\in\{1,2,\ldots,m\}$ is the class label assigned to $\x_i$. Given this set of examples, we learn a classifier $\w$ to predict the label $y$ of $\x\in \R^d$. To learn the classifier, a loss function $\ell(\w;\x,y)$ will be introduced to  penalize the deviation of the prediction of $\w$ on $\x$ from the true label $y$. In this problem setting, our goal is to find an approximate solution of the following optimization problem
\bqs
\min_{\w\in\R^d} P(\w) \qquad P(\w) :=\frac{1}{n}\sum^n_{i=1}\phi_i(\w),
\eqs
where $\phi_i(\w)= \ell(\w,\x_i,y_i) + \frac{\lambda}{2}\|\w\|^2$, and $\lambda$ is a regularization parameter.

To solve the above optimization, a standard method is gradient descent, which can be described by the following update rule for $t=1,2,\ldots$
\[
\w_{t+1}=\w_t-\eta_t\nabla P(\w_t)=\w_t- \frac{\eta_t}{n}\sum^n_{i=1}\nabla\phi_i(\w_t).
\]
However, at each step, gradient descent requires the calculation of $n$ derivatives, which is expensive when $n$ is very large. A popular modification is Stochastic Gradient Descent (SGD): at each iteration $t=1,2,\ldots,$ we draw $i_t$ uniformly randomly from $[n]:=\{1,\ldots,n\}$, and let
\[
\w_{t+1} = \w_t - \eta_t\nabla\phi_{i_t}(\w_t).
\]
Because $\E[\nabla\phi_{i_t}(\w_t)|\w_t]=\frac{1}{n}\sum^n_{i=1}\nabla\phi_{i}(\w_t)=\nabla P(\w_t)$, the expectation $\E[\w_{t+1}|\w_t]$ equals the full gradient. The advantage of stochastic gradient is that each step only relies on a single derivative $\nabla\phi_{i_t}(\w_t)$, and thus the computational cost is $1/n$ that of the standard gradient descent. However, a disadvantage of the method is that the randomness introduces variance, which is caused by the fact that $\nabla\phi_{i_t}(\w_t)$ equals the gradient $\nabla P(\w_t)$ in expectation but each $\nabla\phi_i(\w_t)$ is different. In particular, if it has a large variance, then the convergence will be slow.

For example, consider the case that each $\phi_i(\w)$ is $(1/\gamma)$-smooth, then we have
\bqs
\E P(\w_{t+1})&=&\E P\big(\w_t-\eta_t\nabla\phi_{i_t}(\w_t)\big)\le\E P(\w_t) - \eta_t\|\E \nabla P(\w_t)\|^2 + \frac{\eta_t^2}{2\gamma}\E\|\nabla\phi_{i_t}(\w_t)\|^2\\
&\le& \E P(\w_t) -\eta_t(1-\frac{\eta_t}{2\gamma})\E\|\nabla P(\w_t)\|^2 + \frac{\eta_t^2}{2\gamma}\V(\nabla\phi_{i_t}(\w_t)),
\eqs
where the variance is
\bqs
\V(\nabla\phi_{i_t}(\w_t))=\E\|\nabla\phi_{i_t}(\w_t)-\nabla P(\w_t)\|^2.
\eqs
From the above inequality, we can observe that the smaller the variance, the more reduction on the objective function we have. To reduce the variance, a typical method is to i.i.d. uniformly sample a mini-batch of indices $\B_t=\{\itj\in\{1,2,\ldots,n\}|j=1,2,\ldots,b\}$ with replacement from the set of indices, and then update the classifier as
\bqs
\w_{t+1} =\w_t-\frac{\eta_t}{b}\sum_{s\in\B_t}\nabla\phi_{s}(\w_t),
\eqs
which  equals GD update in expectation. For this update method, the reduction on the objective value can be similarly computed as:
\bqs
\E P(\w_{t+1})\le \E P(\w_t) -\eta_t(1-\frac{\eta_t}{2\gamma})\E\|\nabla P(\w_t)\|^2 + \frac{\eta_t^2}{2\gamma b}\V(\nabla\phi_{i_t^1}(\w_t))
\eqs
using the fact that every $s\in\B_t$ is i.i.d. uniformly sampled from $\{1,2,\ldots,n\}$.
Although the variance is reduced to $1/b$ of the one for using a single uniformly sampled example, the computation becomes $b$ times that of the standard SGD.
Nevertheless, minibatch training is needed to parallelize the SGD algorithm and has been commonly used in practice.
In the next subsection, we will show that the variance of minibatch SGD can be significantly reduced if we employ an appropriate stratified sampling strategy instead of uniform sampling.

\subsection{Algorithm}
\vspace{-0.1in}
The main idea is as follows. At the $t$-th step, we will use some clustering method to separate the training set of indices $[n]$ into $k$ clusters $\C_1^t,\C_2^t,\ldots,\C_k^t$, where
$\C_i^t\subset[n],\ \cup^k_{i=1}\C_i^t=[n],\ \C_j^t\cap\C_i^t=\emptyset$ $\forall i\not =j$ and $|\C_i^t|=n_i^t$
(we also assume that each cluster contains only one class label). Given these $k$ clusters, we will independently sample $k$ subsets $\B_1^t,\B_2^t,\ldots,\B_k^t$ from $\C_1^t,\C_2^t,\ldots,\C_k^t$, respectively, where each index $s\in\B_i^t$ is i.i.d. uniformly sampled from $\C_i^t$, with $|\B_i^t|=b_i^t$, $\sum b_i^t :=b$ and $\B^t=\cup^k_{i=1}\B_i^t$. Given these sampled indices, the proposed algorithm works as follows
\[
\w_{t+1}=\w_t-\frac{\eta_t}{n}\sum^k_{i=1}\frac{n_i^t}{b_i^t}\sum_{s\in\B_i^t}\nabla\phi_s(\w_t),
\]
which equals the GD update in expectation, since
\begin{align*}
\E \; \left[\frac{1}{n}\sum^k_{i=1}\frac{n_i^t}{b_i^t}\sum_{s\in\B_i^t}\nabla\phi_s(\w_t)|\w_t\right]
=&\frac{1}{n}\sum^k_{i=1}\frac{n_i^t}{b_i^t}\sum_{s\in\B_i^t}\E_t\nabla\phi_s(\w_t)
\\
=& \frac{1}{n}\sum^k_{i=1}\frac{n_i^t}{b_i^t}\sum_{s\in\B_i^t}\frac{1}{n_i^t}\sum_{s\in\C_i^t}\nabla\phi_s(\w_t)=\nabla P(\w_t).
\end{align*}

Assume the update methods for $\C_1^t,\C_2^t,\ldots,\C_k^t$ and $b_1^t,b_2^t,\ldots,b_k^t$ are provided, we can summarize the proposed method in Algorithm~\ref{alg:SGD-ss}.
\begin{algorithm}[htpb]
\caption{Stochastic Gradient Descent with Stratified Sampling (SGD-ss)} \label{alg:SGD-ss}
\begin{algorithmic}
\STATE {\bf Input}: Clusters Number $k$, Minibatch Size $b$.
\STATE {\bf Initialize}:  $\w_1=0$, $\{\C_1^1,\C_2^1,\ldots,\C_k^1\}$, $b_1^1,b_2^1,\ldots,b_k^1$, such that each $\C_i^1$ contains only one label.
\FOR{$t=1,2,\ldots,T$}
\FOR{$s=1,2,\ldots,k$}
\STATE $\B^t_s =\emptyset$;
\FOR{$r=1,2\ldots,b_s^t$}
\STATE Uniformly sample $i_r\in\C_s^t$ and $\B_s^t = \B_s^t \cup \{i_r\}$;
\ENDFOR
\ENDFOR
\STATE Update $\w_{t+1}=\w_t -  \frac{\eta_t}{n}\sum^k_{i=1}\frac{n_i^t}{b_i^t}\sum_{s\in\B_i^t}\nabla\phi_s(\w_t)$;
\STATE Update $\C_1^{t+1},\C_2^{t+1},\ldots,\C_k^{t+1}$ such that each $\C_i^{t+1}$ contains only one label;
\STATE Update $b_1^{t+1},b_2^{t+1},\ldots,b_k^{t+1}$;
\ENDFOR
\end{algorithmic}
\end{algorithm}

For the proposed algorithm, the reduction on the objective function can be similarly computed as follows:
\bqs
\E P(\w_{t+1})
&\le& \E P(\w_t) - \eta_t(1-\frac{\eta_t}{2\gamma}) \E\|\nabla P(\w_t)\|^2 + \frac{\eta_t^2}{2\gamma}V(\C_1^t,\ldots,\C_k^t,b_1^t,\ldots,b_k^t),
\eqs
where $V(\C_1^t,\ldots,\C_k^t,b_1^t,\ldots,b_k^t)=\V\Big(\frac{1}{n}\sum^k_{i=1}\frac{n_i^t}{b_i^t}\sum_{s\in\B_i^t}\nabla\phi_s(\w_t)\Big)$.

According to the above analysis, to maximize the reduction on the objective function, the optimal clusters and minibatch distribution can be obtained by solving the following optimization problem:
\bq
\label{eqn:variance-t}
\min_{\C_1^t,\ldots,\C_k^t,b_1^t,\ldots,b_k^t}V(\C_1^t,\ldots,\C_k^t,b_1^t,\ldots,b_k^t)&=&\frac{1}{n^2}\sum^k_{i=1}(\frac{n_i^t}{b_i^t})^2\sum_{s\in\B_i^t}\V(\nabla\phi_s(\w_t)) \nonumber\\
&=&\frac{1}{n^2}\sum^k_{i=1}\frac{n_i^t}{b_i^t}\sum_{s\in\C_i^t}\|\nabla\phi_s(\w_t)-\frac{1}{n_i^t}\sum_{r\in\C_i^t}\nabla\phi_r(\w_t)\|^2,
\eq
which can be considered as a dynamically weighted $k$-means problem, where the weights of gradients in the same cluster are the same and optimized with the clusters simultaneously.

Although, this sampling method can minimize the variance of the stochastic estimator, it requires the calculation of $n$ derivatives and requires running a clustering algorithm for every iteration, which is clearly impractical. To address this issue, we assume that $\partial_\w\ell(\w;\x,y)$ is  $L$-Lipschitz in $\x$ for fixed $\w$ and $y$. Now assume further that each cluster contains only one class label: $y_s=y_r$ $\forall s,r\in\C^t_i$.
Under these assumptions, we can relax the previous expression of variance as follows:
\bqs
&&\hspace{-0.3in}V(\C_1^t,\ldots,C_k^t,b_1^t,\ldots,b_k^t)\\
&&\hspace{-0.3in}=\frac{1}{n^2}\sum^k_{i=1}\frac{n_i^t}{b_i^t}\sum_{s\in\C_i^t}\|\partial_\w\ell(\w_t;\x_s,y_s)-\frac{1}{n_i^t}\sum_{r\in\C_i^t}\partial_\w\ell(\w_t;\x_r,y_r)\|^2\\
&&\hspace{-0.3in}\leq \frac{1}{n^2}\sum^k_{i=1}\frac{n_i^t}{b_i^t}\sum_{s\in\C_i^t} \|\partial_\w\ell(\w_t;\x_s,y_s)-\partial_\w\ell(\w_t;\mu_i^t,y_s)\|^2
\\
&&\hspace{-0.3in}
\le \frac{L^2}{n^2}\sum^k_{i=1}\frac{n_i^t}{b_i^t} \sum_{s\in\C_i^t}\|\x_s-\mu_i^t\|^2 ,
\eqs
where $\mu_i^t=\sum_{s\in\C_i^t}\x_s/n^t_i$. This relaxation inspires us to find an iteration-independent sampling strategy where $\C_i^t=\C_i$, $n_i^t=n_i$, and $b_i^t=b_i$, $\forall t,i$, by solving the following optimization problem,
which corresponds to an upperbound of \eqref{eqn:variance-t}
\bq
\label{eqn:predefined-stratified-sampling}
\min_{\C_i,b_i, y_s=y_r,\forall s,r\in\C_i}\frac{L^2}{n^2}\sum^k_{i=1}\frac{n_i}{b_i} \sum_{s\in\C_i}\|\x_s-\frac{1}{n_i}\sum_{r\in\C_i} \x_r\|^2 .
\eq
The solution of this optimization problem can be pre-calculated and used at every iteration. Given $\C_1,\ldots,\C_k$, it is easy to verify that the optimal $b_1,\ldots,b_k$ can be calculated as ($b_i$ is relaxed to take non-integer values): 
\bq\label{eqn:minibatch-size-lipschitz}
b_i = \frac{b n_i\sqrt{v_i}}{\sum^n_{j=1}n_j\sqrt{v_j}},\quad v_i= \frac{1}{n_i}\sum_{s\in\C_i}\|\x_s -\frac{1}{n_i}\sum_{r\in\C_i} \x_r\|^2 ,
\eq
and we can simplify \eqref{eqn:predefined-stratified-sampling} to the following optimization problem
\bq\label{eqn:predefined-stratified-sampling-C}
\min_{\C_i, y_s=y_r,\forall s,r\in\C_i}\sum^k_{i=1}n_i \sqrt{\frac{1}{n_i} \sum_{s\in\C_i}\|\x_s-\frac{1}{n_i}\sum_{r\in\C_i} \x_r\|^2} .
\eq
which can be solved by a $k$-means style alternating optimization algorithm. An even simpler method is to use the standard $k$-means algorithm separately for each class label to obtain the clusters $\{\C_i\}$.

\subsection{Analysis}
This section provides a convergence analysis of the proposed algorithm. Before presenting the results, we introduce the notation:
\[
\w^*=\min_{\w}P(\w),
\]
which implies $\w^*$ is the optimal solution. Given this notation, we  begin our analysis with a technical lemma.
\begin{lemma}\label{lem:progress-t}
Suppose $P(\w)$ is $H$-strongly convex and $(1/\gamma)$-smooth over $\R^d$. If $\eta_t\in (0, \gamma]$, then the proposed algorithm satisfies the following inequality for any $t\ge 1$,
\bqs
\E \Big[P(\w_{t+1})-P(\w^*)\Big]\le \frac{1}{2\eta_t}\E[\|\w_t-\w^*\|^2 - \|\w_{t+1}-\w^*\|^2]-\frac{H}{2}\E\|\w_t-\w^*\|^2 + \eta_t \E V_t,
\eqs
where
\[
V_t = \V\Big(\frac{1}{n}\sum^k_{i=1}\frac{n_i^t}{b_i^t}\sum_{s\in\B_i^t}\nabla\phi_s(\w_t)\Big).
\]
\end{lemma}
\begin{proof}
To simplify the analysis, we denote $\g_t=\frac{1}{n}\sum^k_{i=1}\frac{n_i^t}{b_i^t}\sum_{s\in\B_i^t}\nabla\phi_s(\w_t)$, and
\[
\delta_t=\Big\langle \frac{1}{n}\sum^k_{i=1}\frac{n_i^t}{b_i^t}\sum_{s\in\B_i^t}\nabla\phi_s(\w_t),   \w_t-\w^*\Big\rangle-\Big[ \frac{1}{n}\sum^k_{i=1}\frac{n_i^t}{b_i^t}\sum_{s\in\B_i^t}(\phi_s(\w_t)-\phi_s(\w^*))+\frac{H}{2}\|\w_t-\w^*\|^2\Big].
\]
Given these notations, we can first derive
\bqs
&&\hspace{-0.2in} \|\w_t-\w^*\|^2 - \|\w_{t+1}-\w^*\|^2=\|\w_t-\w^*\|^2 - \big\|\w_t - \eta_t\g_t-\w^*\big\|^2= 2\Big\langle \eta_t\g_t,   \w_t-\w^*\Big\rangle
- \big\|\eta_t\g_t \big\|^2\\
&&\hspace{-0.2in}= 2\eta_t \delta_t+2\eta_t\Big[ \frac{1}{n}\sum^k_{i=1}\frac{n_i^t}{b_i^t}\sum_{s\in\B_i^t}(\phi_s(\w_t)-\phi_s(\w^*))+\frac{H}{2}\|\w_t-\w^*\|^2\Big]-\|\eta_t \g_t \|^2.
\eqs
Taking expectation of the above equality, and using the fact
\bqs
\E \delta_t=\Big\langle\nabla P(\w_t),   \w_t-\w^*\Big\rangle-\Big[P(\w_t)-P(\w^*)+\frac{H}{2}\|\w_t-\w^*\|^2\Big]\ge 0,
\eqs
we can obtain
\bqs
&&\hspace{-0.3in}\E[\|\w_t-\w^*\|^2 - \|\w_{t+1}-\w^*\|^2] \ge 2\eta_t\E \Big[P(\w_t)-P(\w^*)+\frac{H}{2}\|\w_t-\w^*\|^2\Big]-\eta_t^2\E\|\g_t\|^2.
\eqs
In addition, using the fact that $P$ is $(1/\gamma)$-smooth and $\w_{t+1}=\w_t - \eta_t \g_t$, we can derive
\bqs
 \E P(\w_{t+1}) \le \E P(\w_t) - \eta_t \E \|\nabla P(\w_t)\|^2 + \frac{\eta^2_t}{2\gamma}\E\|\g_t\|^2.
\eqs
Combining the above two inequalities, we can get
\bqs
&&\E[P(\w_{t+1})-P(\w^*)]\\
&&\le\frac{1}{2\eta_t}\E[\|\w_t-\w^*\|^2 \hspace{-0.03in}- \|\w_{t+1}-\w^*\|^2]-\frac{H}{2}\E\|\w_t-\w^*\|^2 \hspace{-0.03in}+ (\frac{\eta_t}{2}+\frac{\eta_t^2}{2\gamma})\E\|\g_t\|^2-\eta_t \E\|\nabla P(\w_t)\|^2.
\eqs
Combing the above inequality with the facts $\E\|\g_t\|^2 =\V(\g_t) + (\|\E \g_t\|)^2=\V(\g_t) + \|\nabla P(\w_t)\|^2$, and $\eta_t\in (0,\gamma]$ will conclude the the proof of this lemma.
\end{proof}

Given the above lemma, we will prove a convergence result for the proposed algorithm, when  $P(\w)$ is $H$-strongly convex and $(1/\gamma)$-smooth.
\begin{thm}
Suppose $P(\w)$ is $H$-strongly convex and $(1/\gamma)$-smooth. If we set $\eta_t=1/(a + H t)$ where $a\ge 1/\gamma -H$, then the proposed algorithm will satisfy the following inequality for all $T$,
\bqs
\inf_{t\in[T]}\E P(\w_{t+1})- P(\w^*)\le \frac{1}{T}\sum^T_{t=1}\E P(\w_{t+1})-P(\w^*)\le\frac{1}{T}\Big[\frac{a}{2}\|\w^*\|^2  +\E\sum^T_{t=1}\frac{V_t}{a+Ht}\Big].
\eqs
where $V_t=\V\Big(\frac{1}{n}\sum^k_{i=1}\frac{n_i^t}{b_i^t}\sum_{s\in\B_i^t}\nabla\phi_s(\w_t)\Big)$.
\end{thm}
\begin{proof}
Firstly, it is easy to verify $\eta_t\in (0,\gamma]$, $\forall t\ge 1$. Because $P(\w)$ and $\eta_t$ satisfy the assumptions in Lemma 1,  we have
\bqs
\E \Big[P(\w_{t+1})-P(\w^*)\Big]\le \frac{1}{2\eta_t}\E[\|\w_t-\w^*\|^2 - \|\w_{t+1}-\w^*\|^2]-\frac{H}{2}\E\|\w_t-\w^*\|^2 + \eta_t \E V_t,
\eqs
Summing the above inequality over $t=1,2,\ldots,T$, and using $\eta_t = 1/(a+Ht)$, we get
\bqs
&&\hspace{-0.3in}\sum^T_{t=1}\E P(\w_{t+1})-\sum^T_{t=1}P(\w^*)\\
&&\hspace{-0.3in}\le \sum^T_{t=1}\frac{a+Ht}{2}\E[\|\w_t-\w^*\|^2 - \|\w_{t+1}-\w^*\|^2]-\frac{H}{2}\sum^T_{t=1}\E\|\w_t-\w^*\|^2 + \E\sum^T_{t=1}\frac{V_t}{a+Ht}\\
&&\hspace{-0.3in}= \frac{a}{2}\E\|\w_1-\w^*\|^2 - \frac{a+HT}{2}\E\|\w_{T+1}-\w^*\|^2 + \E\sum^T_{t=1}\frac{V_t}{a+Ht}\le\frac{a}{2}\|\w^*\|^2  +\E\sum^T_{t=1}\frac{V_t}{a+Ht}.
\eqs
Dividing the above inequality with $T$ will conclude the theorem.
\end{proof}
If $V_t=0$ for all $t\in[T]$, then the above theorem will give a $O(1/T)$ convergence bound, which is the same as the convergence bound of gradient descent for a convex ($0$-strongly convex) and smooth objective function. However, if the objective function is $H$-strongly convex with $H>0$, the convergence bound of gradient descent is $O(c^T)$ with $c\in(0,1)$, which is significantly better than the bound in this theorem.
Therefore in addition to the above theorem, we also prove a linear convergence bound when $V_t=0,\ \forall t\in[T]$. 
First we prove a technical lemma as follows.
\begin{lemma}
Suppose $P(\w)$ is $H$-strongly convex and $(1/\gamma)$-smooth over $\R^d$. Then we have the following inequality for any $\u,\v\in\R^d$,
\bqs
\langle\nabla P(\u)-\nabla P(\v), \u-\v\rangle \ge \frac{H/\gamma}{H + 1/\gamma}\|\u-\v\|^2 + \frac{1}{H + 1/\gamma}\|\nabla P(\u) - \nabla P(\v)\|^2.
\eqs
\end{lemma}
\begin{proof}
To simplify the notations, we define
\bqs
f(\w) = P(\w) - \frac{H}{2}\|\w\|^2.
\eqs
Firstly, $f(\w)$ is convex ($0$-strongly convex), since the following inequality holds for any $\u,\v\in\R^d$,
\bqs
f(\u)-[f(\v)+\langle\nabla f(\v), \u-\v\rangle]&=&P(\u)-\frac{H}{2}\|\u\|^2-[P(\v)-\frac{H}{2}\|\v\|^2+\langle\nabla P(\v)-H\v, \u-\v\rangle]\\
&=&P(\u)-[P(\v) +\langle\nabla P(\v), \u-\v\rangle +\frac{H}{2}\|\u-\v\|^2]\ge 0,
\eqs
where the final inequality used the $H$-strongly convexity of $P(\w)$.

Secondly, $f(\w)$ is $(1/\gamma - H)$-smooth, since the following inequality holds for any $\u,\v\in\R^d$,
\bqs
&&f(\u) - [f(\v)+\langle\nabla f(\v), \u-\v\rangle +\frac{1/\gamma-H}{2}\|\u-\v\|^2]\\
&&=P(\u)-\frac{H}{2}\|\u\|^2-[P(\v)-\frac{H}{2}\|\v\|^2+\langle\nabla P(\v)-H\v, \u-\v\rangle+ \frac{1/\gamma-H}{2}\|\u-\v\|^2]\\
&&=P(\u) - [P(\v)+\langle\nabla P(\v), \u-\v\rangle +\frac{1}{2\gamma}\|\u-\v\|^2]\le 0,
\eqs
where the final inequality used the fact $P(\w)$ is $(1/\gamma)$-smooth.

Because $f(\w)$ is convex and $(1/\gamma-H)$-smooth, according to the Property 1, the co-coercivity of $f(\w)$ with parameter $\frac{1}{1/\gamma - H}$ gives
\bqs
\langle (\nabla P(\u)-H\u)- (\nabla P(\v)-H\v), \u-\v\rangle \ge \frac{1}{1/\gamma-H} \|(\nabla P(\u)-H\u)-(\nabla P(\v)-H\v)\|^2.
\eqs
Re-arranging the above inequality conclude the proof of this lemma.
\end{proof}
Given the above lemma, we can prove another bound for the proposed algorithm as follows, when  $P(\w)$ is  $H$-strongly convex and $(1/\gamma)$-smooth.
\begin{thm}
Suppose $P(\w)$ is  $H$-strongly convex and $(1/\gamma)$-smooth. If we set $\eta_t =\eta\in(0, \frac{2}{H+1/\gamma}]$, then the proposed algorithm will satisfy the following inequality for all $T$,
\bqs
\E P(\w_{T+1})-P(\w^*)\le \frac{\alpha(\eta)^T}{2\gamma}\|\w^*\|^2 +\frac{\eta^2}{2\gamma}\sum^T_{t=1}\alpha(\eta)^{T-t}\E V_t,
\eqs
where $\alpha(\eta)=1- \frac{2 \eta H/\gamma}{H + 1/\gamma}\in[(\frac{H-1/\gamma}{H+1/\gamma})^2, 1)$ and $V_t=\V\Big(\frac{1}{n}\sum^k_{i=1}\frac{n_i^t}{b_i^t}\sum_{s\in\B_i^t}\nabla\phi_s(\w_t)\Big)$.
\end{thm}

\begin{proof}
Using the fact $\w_{t+1} = \w_t - \eta \g_t$ where $\g_t=\frac{1}{n}\sum^k_{i=1}\frac{n_i^t}{b_i^t}\sum_{s\in\B_i^t}\nabla\phi_s(\w_t)$, we have
\bqs
&&\|\w_{t+1}-\w^*\|^2 =\|\w_t-\eta\g_t -\w^*\|^2=\|\w_t-\w^*\|^2 - 2\eta\langle \g_t, \w_t-\w^*\rangle + \eta^2 \|\g_t\|^2.
\eqs
Taking expectation of the above inequality, using $\E\|\g_t\|^2=\V(\g_t)+\|\E\g_t\|^2=\V(\g_t)+\|\nabla P(\w_t)\|^2$ and using Lemma 2, will derive
\bqs
&&\E \|\w_{t+1}-\w^*\|^2 = \E\|\w_t-\w^*\|^2 - 2\eta \E \langle\nabla P(\w_t), \w_t-\w^*\rangle + \eta^2\E[V_t +\|\nabla P(\w_t)\|^2]\\
&&\le (1- \frac{2 \eta H/\gamma}{H + 1/\gamma})\E\|\w_t-\w^*\|^2 + \eta(\eta- \frac{2}{H + 1/\gamma})\E\|\nabla P(\w_t)\|^2 + \eta^2 \E V_t\\
&&\le \alpha(\eta) \E\|\w_t-\w^*\|^2  + \eta^2 \E V_t,
\eqs
where the final inequality used $\eta\in(0, \frac{2}{H+1/\gamma}]$.
Using the above inequality iteratively, we can derive
\bqs
\E \|\w_{T+1}-\w^*\|^2 \hspace{-0.1in}&\le&\hspace{-0.1in} \alpha(\eta) \E\|\w_T-\w^*\|^2  + \eta^2 \E V_T\le \alpha(\eta)\Big[ \alpha(\eta)\E\|\w_{T-1}-\w^*\|^2 + \eta^2\E V_{T-1}\Big]+ \eta^2 \E V_T\\
&\le&\ldots\le \alpha(\eta)^T\E\|\w_{1}-\w^*\|^2 +\alpha(\eta)^{T-1}\eta^2\E V_1 +\ldots +\alpha(\eta) \eta^2 \E V_{T-1}+ \eta^2 \E V_T\\
&=& \alpha(\eta)^T\E\|\w_{1}-\w^*\|^2 +\eta^2\sum^T_{t=1}\alpha(\eta)^{T-t}\E V_t .
\eqs
In addition using the facts $P(\w)$ is $(1/\gamma)$-smooth and $\nabla P(\w^*)=0$, we get
\bqs
P(\w)-P(\w^*)\le\langle\nabla P(\w^*), \w-\w^*\rangle + \frac{1}{2\gamma}\|\w-\w^*\|^2=\frac{1}{2\gamma}\|\w-\w^*\|^2.
\eqs
Combining the above two inequalities concludes the proof.
\end{proof}
If $V_t=0$ for all $t\in[T]$, then the above theorem will given a convergence bound of $O(\alpha(\eta)^T)$ with $\alpha(\eta)\in[(\frac{H-1/\gamma}{H+1/\gamma})^2, 1)$, which is the same with convergence bound of gradient descent for a $H$-strongly ($H>0$) convex and smooth objective function.

\section{Experiments}
\label{sec:experiment}
In this section, we evaluate the empirical performance of the proposed algorithm by comparing it to SGD with uniform sampling.

\subsection{Experimental Setup}
To extensively examine the performance, we test all the algorithms on a number of benchmark datasets from web machine learning repositories. Table~\ref{table:datasets} shows the details of the datasets used in our experiments. All of these datasets can be downloaded from the LIBSVM website~\footnote{\url{http://www.csie.ntu.edu.tw/~cjlin/libsvmtools/}}. These datasets were chosen fairly randomly in order to cover various sizes of datasets.
\begin{table}[htpb]
\renewcommand*\arraystretch{1.0}
\begin{center}
\caption{
Details of the datasets in our experiments.}\label{table:datasets}
\begin{tabular}{|l|r|r|r|r|c|}        \hline
Dataset  &$\#$ class & Training size & Testing size  &     $\#$ features & minibatch size\\
\hline\hline
covtype.binary &2  & 523,124  & 57,888  & 54  & 10 \\
letter         &26 &  15,000  &  5,000  & 16  & 26\\
mnist          &10 &  60,000  & 10,000  & 780 & 10\\
pendigits      &10 &   7,494  &  3,498  & 16  & 13 \\
usps           &10 &   7,291  &  2,007  & 256 & 48
\\\hline
\end{tabular}
\end{center}
\renewcommand*\arraystretch{1.0}
\end{table}

To make a fair comparison, all algorithms adopted the same setup in our experiments. In particular,  we performed L2-regularized multiclass logistic regression (convex optimization) to tackle these classification tasks. The regularization parameters of multiclass logistic regression is set as $10^{-5}$, $10^{-4}$, $10^{-4}$, $10^{-3}$ and $10^{-3}$ for \emph{covtype.binay}, \emph{letter},  \emph{mnist}, \emph{pendigits}, and  \emph{usps} respectively. For SGD with uniform sampling and stratified sampling, the step size for the $t$-th round is set as $\eta_t=1/(\lambda t)$ for all the datasets. The minibatch sizes, 
given in Table!\ref{table:datasets},
were chosen as the same value for SGD and SGD-ss on every dataset. For SGD-ss, we used a stratified sampling strategy by solving the optimization problems~\eqref{eqn:minibatch-size-lipschitz} and ~\eqref{eqn:predefined-stratified-sampling-C}.
Note that instead of  using~\eqref{eqn:predefined-stratified-sampling-C},
the clusters $\{\C_i\}$ can also be obtained via the simpler $k$-means method without performance degradation.

All the experiments were conducted by fixing 5 different random seeds for each datasets. All the results were reported by averaging over these 5 runs. We evaluated the algorithms' performance by measuring \emph{primal objective value } on training dataset, i.e., $P(\w_t)$. In addition, to examine the performance of resulting classifiers on test datasets, we also evaluated the \emph{test error rate}. Finally, we also compared the actual variances of stochastic gradient estimators of the two algorithms to check whether the proposed algorithm can effectively reduce the variance.

\subsection{Performance Evaluation}
\begin{figure*}[hptb]
\begin{center}
\includegraphics[width=2.1in]{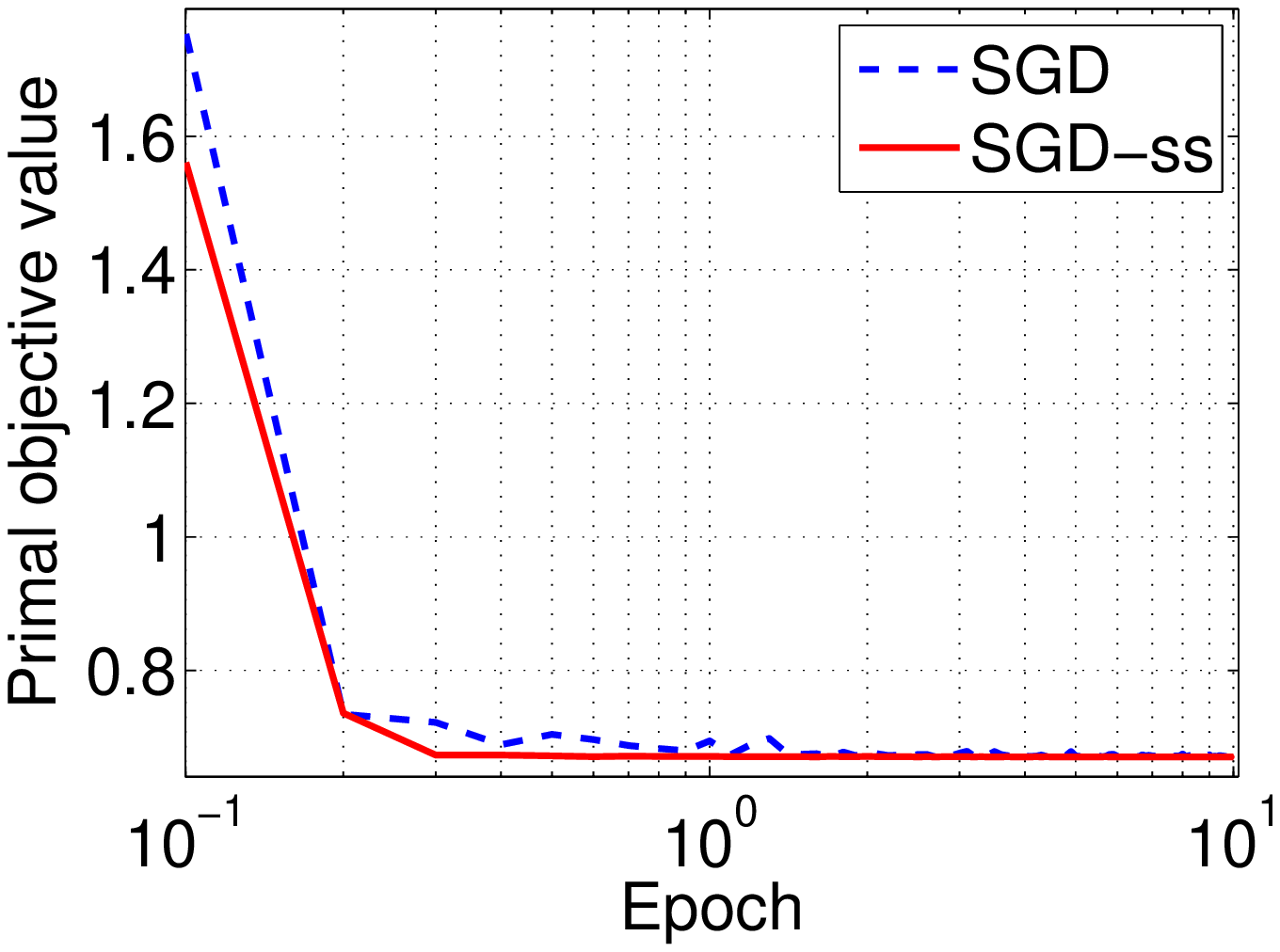}
\includegraphics[width=2.1in]{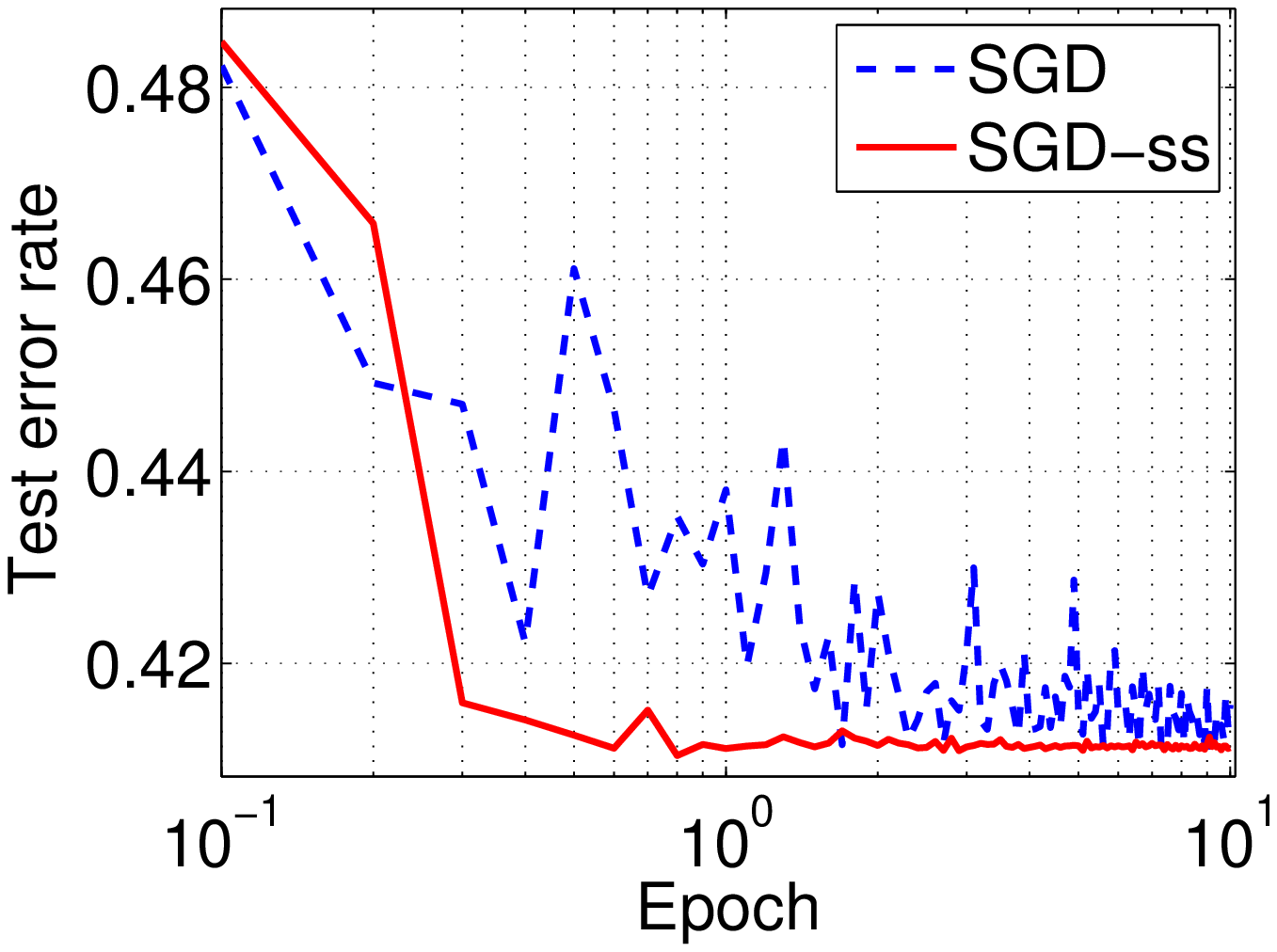}
\includegraphics[width=2.1in]{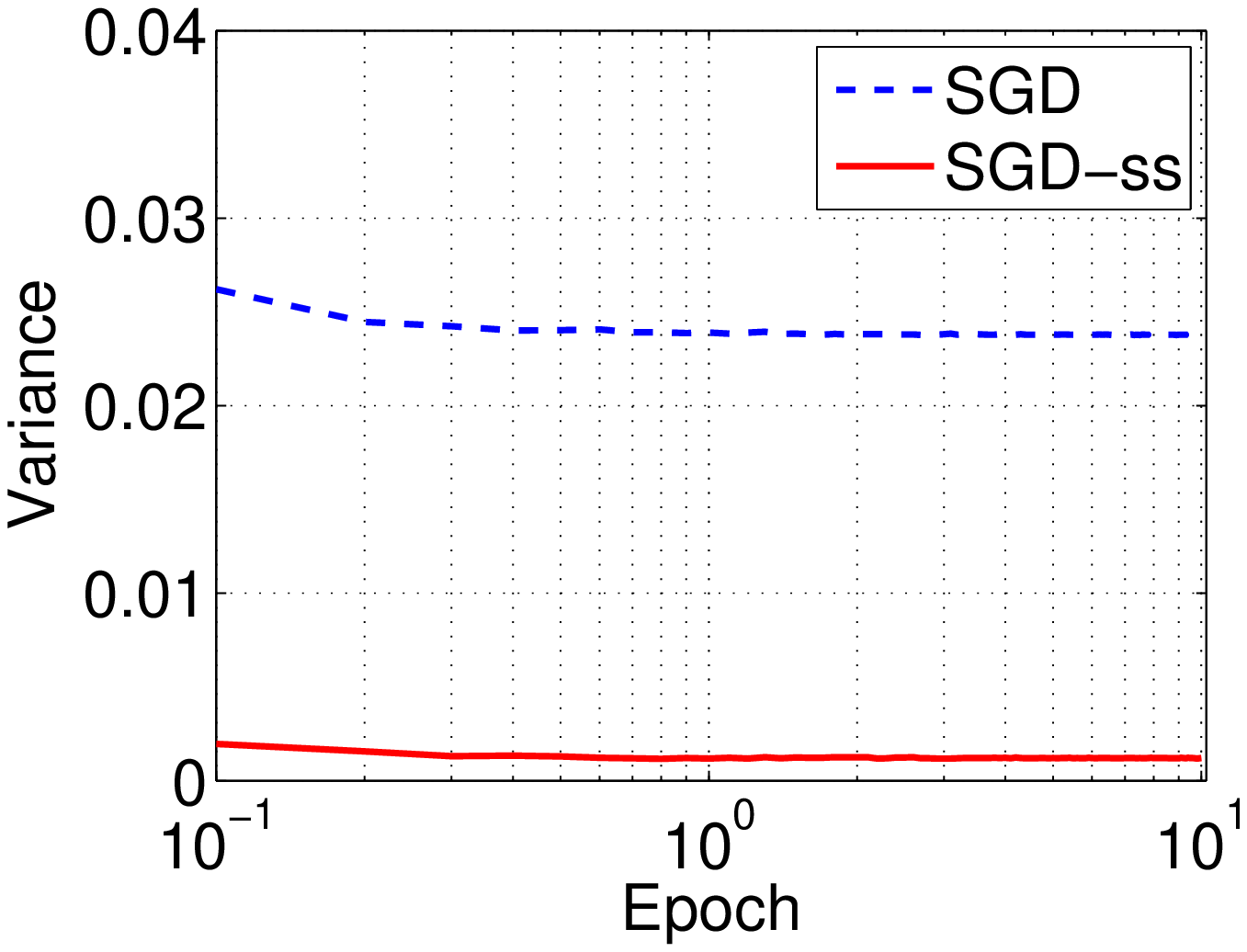}
{\scriptsize \makebox[2.1in]{(a)~Primal objective value }~\makebox[2.1in]{(b)~Test error rate }~\makebox[2.1in]{(c)~Variance }}
\end{center}\vspace{-0.15in}
\caption{Multiclass logistic regression (convex) on {\bf covtype.binary}. Epoch for the horizontal axis
is the number of iterations times the minibatch size divided by the training data size.}
\label{fig:covtype.binary}
\end{figure*}

Figure~\ref{fig:covtype.binary} summarizes the experimental results on the dataset \emph{covtype.binary}. First, Figure~\ref{fig:covtype.binary} (a) shows the primal objective values of SGD-ss in comparison to that of SGD with uniform sampling. We can observe that SGD-ss converges faster and is much more stable than SGD. Because these two algorithm adopted the same minibatch size and learning rates, this observation clearly implies that the proposed stratified sampling strategy is more effective to reduce the variance of the stochastic gradient estimator than uniform sampling. Second, Figure~\ref{fig:covtype.binary} (b) provides test error rates of the two algorithms, where we observe that SGD-ss achieves significantly smaller and stable test error rates than that of SGD. This shows that the proposed stratified sampling approach is effective in improving the testing performance and reducing its variance. Third, according to Figure~\ref{fig:covtype.binary} (c), we observe that the variance for SGD-ss is significantly smaller than that of SGD with uniform sampling. This again demonstrates the effectiveness of the proposed sampling strategy to reduce the variance of the unbiased stochastic gradient estimators.

\begin{figure*}[hptb]
\begin{center}
\includegraphics[width=1.55in]{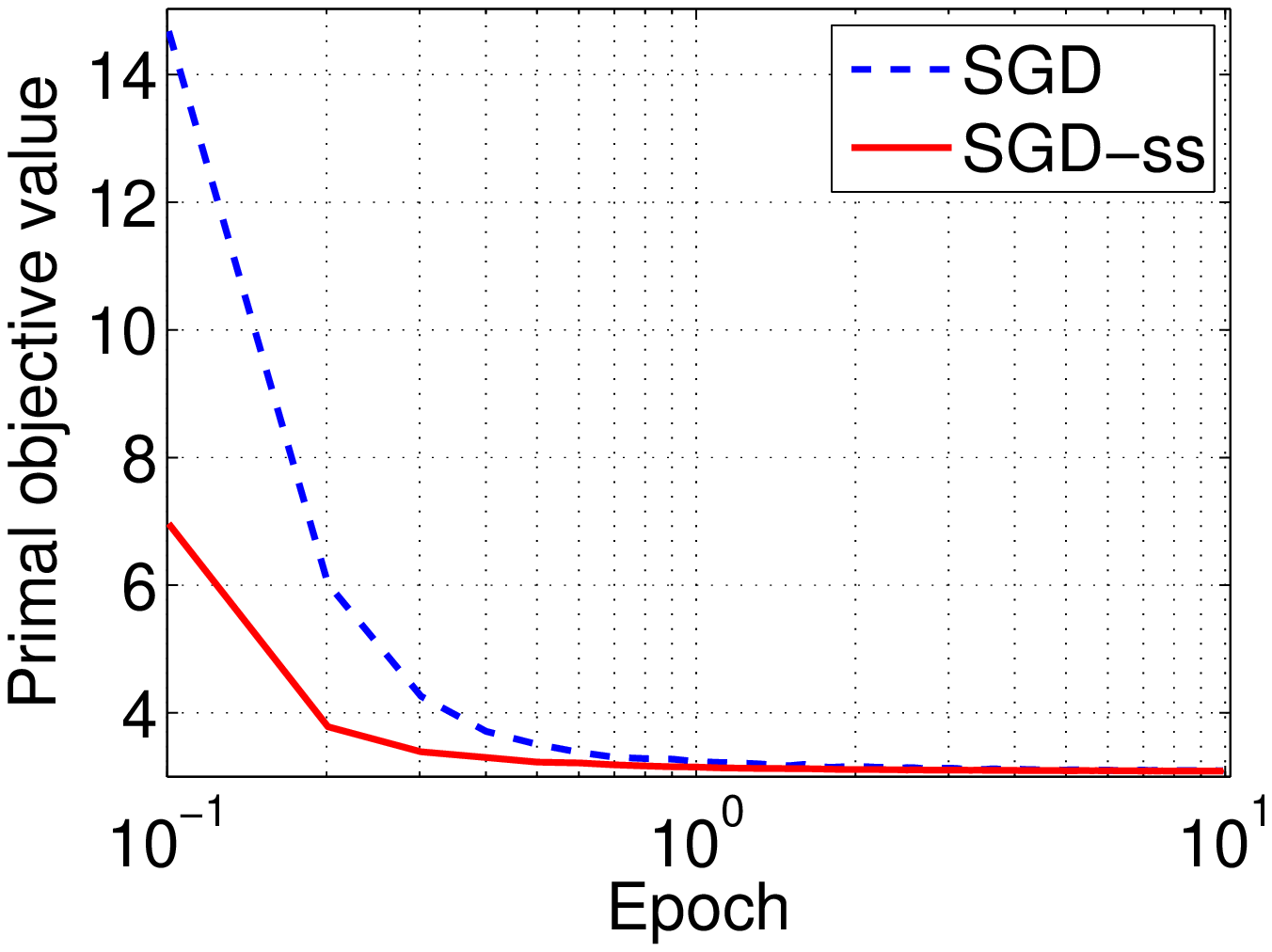}
\includegraphics[width=1.55in]{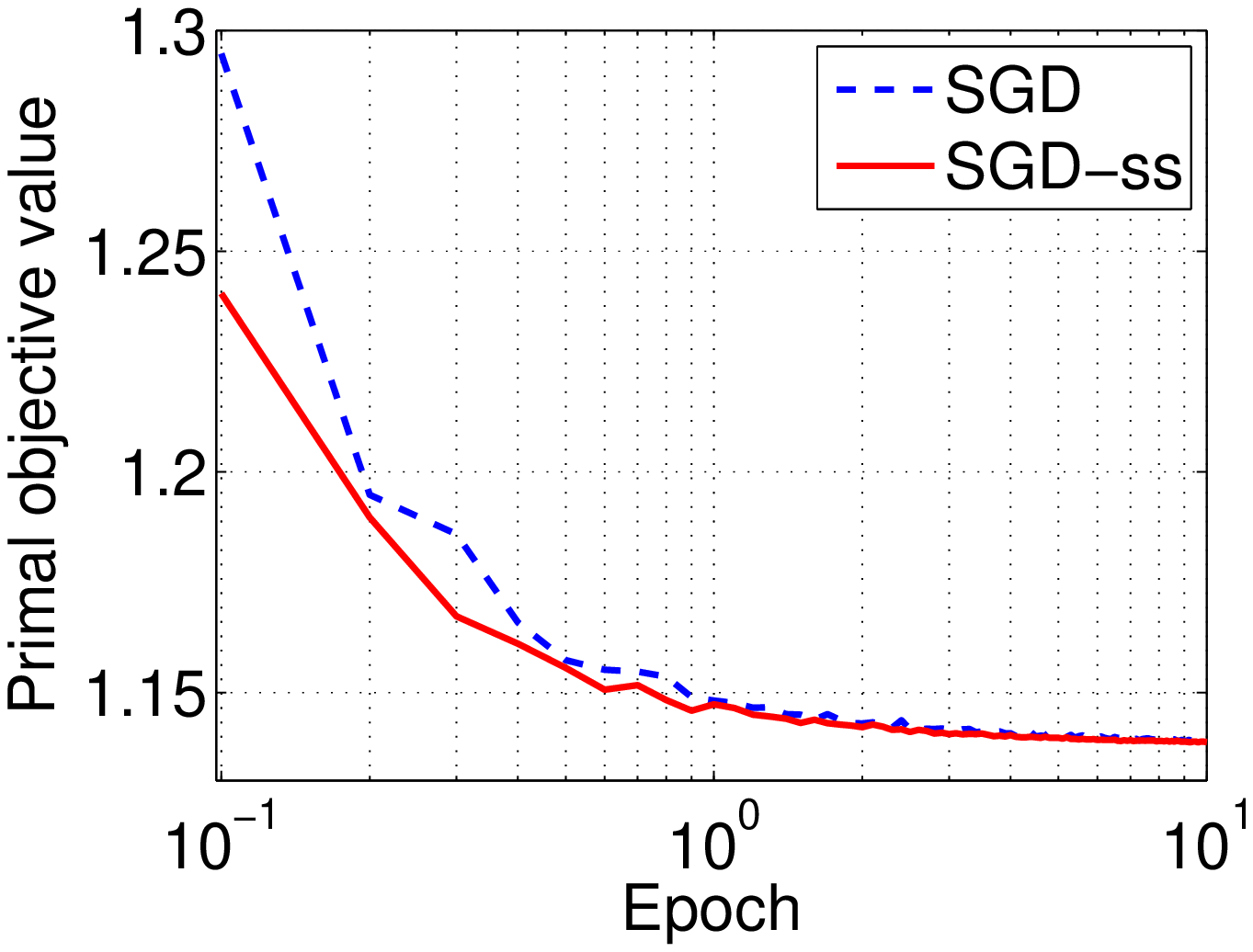}
\includegraphics[width=1.55in]{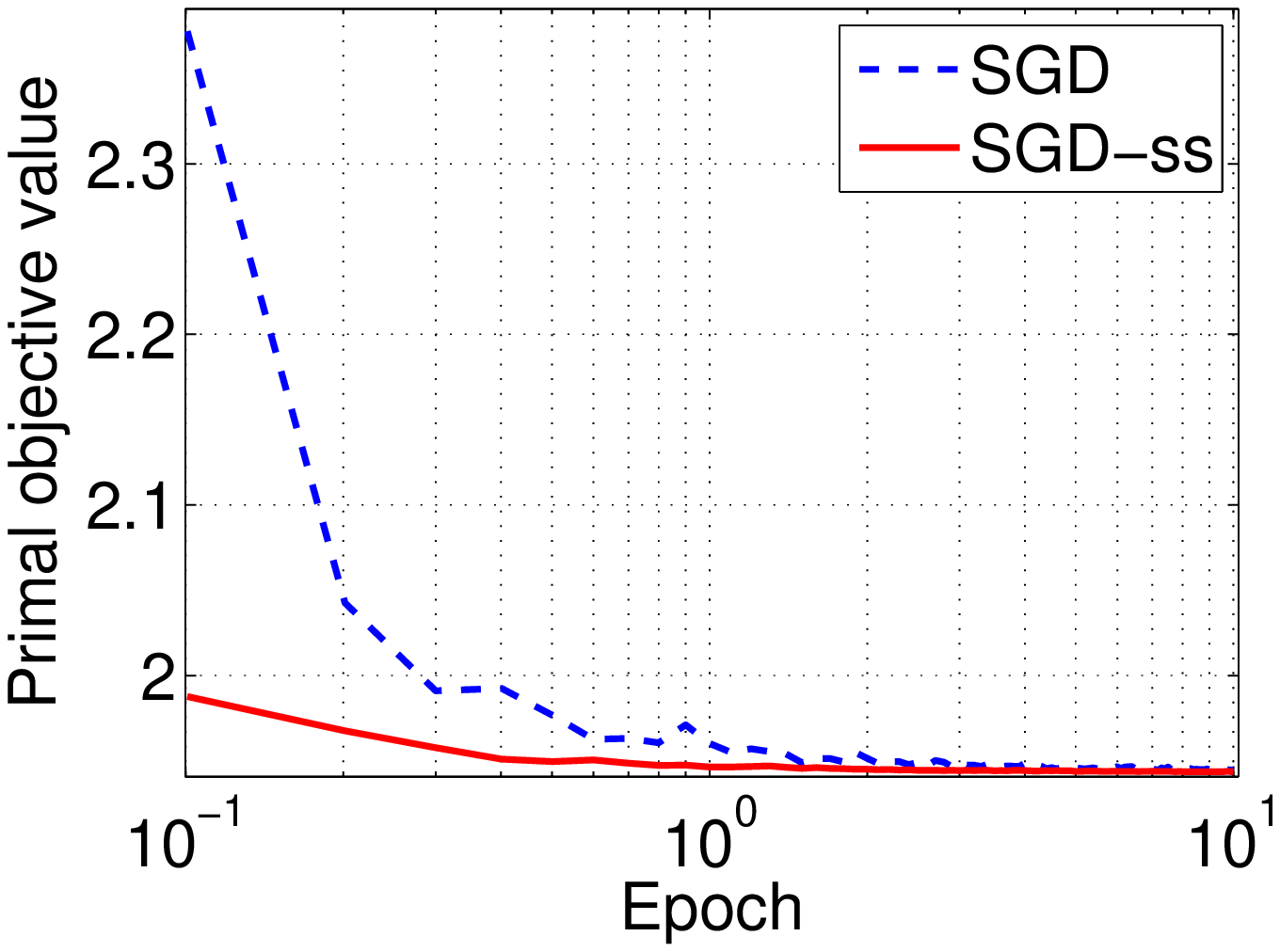}
\includegraphics[width=1.55in]{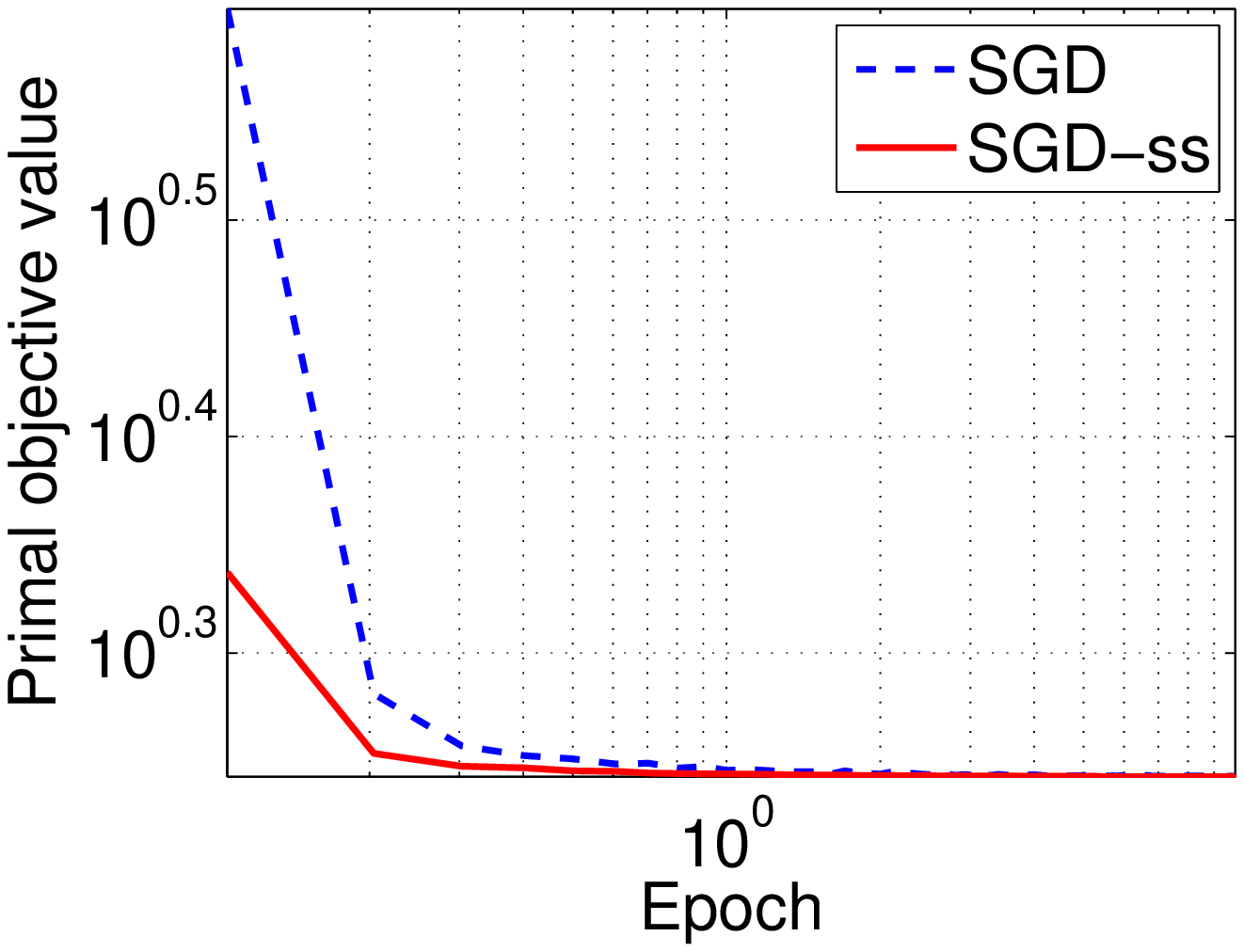}
{\scriptsize \makebox[1.55in]{(a)~letter }~\makebox[1.55in]{(b)~mnist }~\makebox[1.55in]{(c)~pendigits }~\makebox[1.55in]{(d)~ usps}}
\end{center}\vspace{-0.15in}
\begin{center}
\includegraphics[width=1.55in]{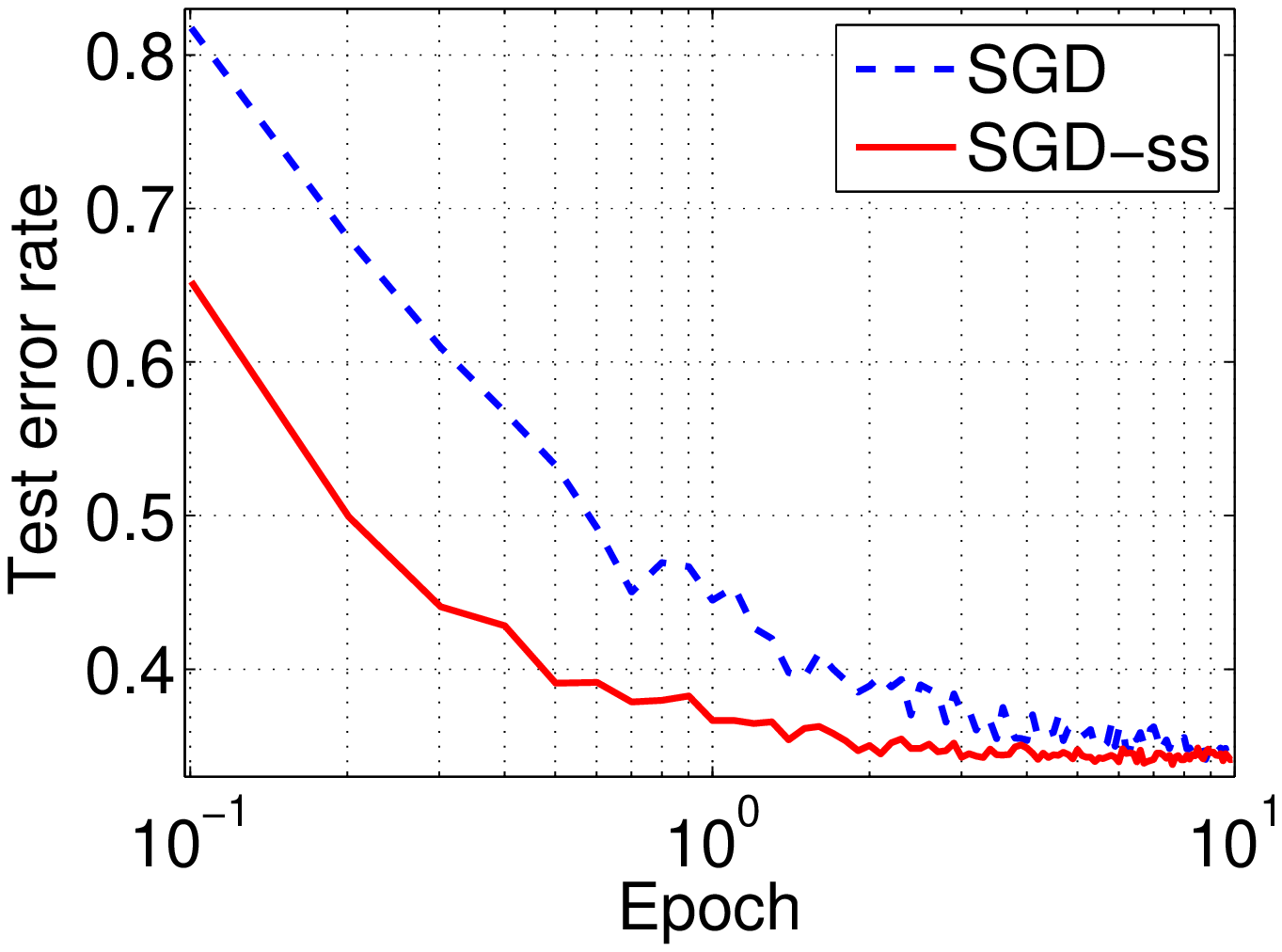}
\includegraphics[width=1.55in]{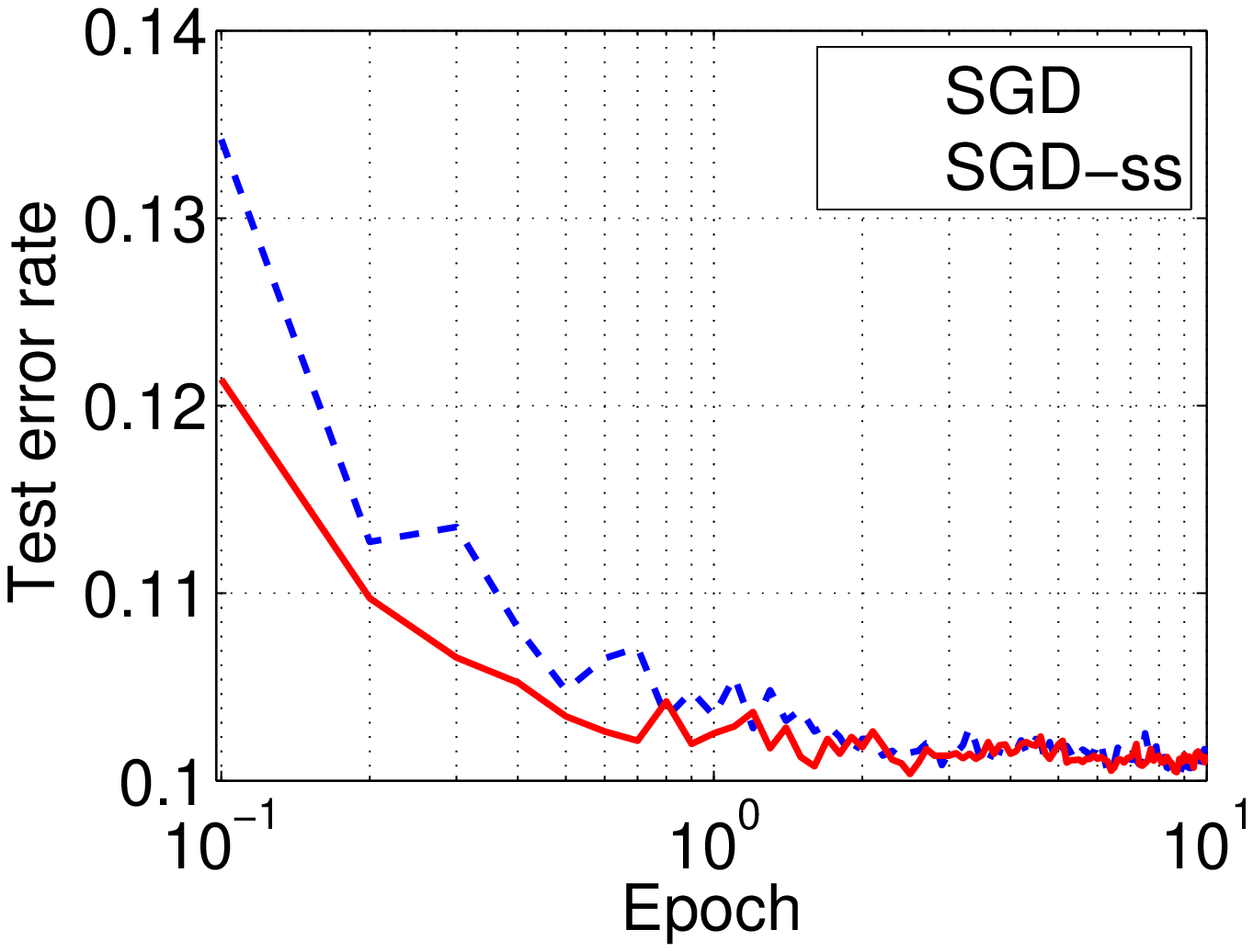}
\includegraphics[width=1.55in]{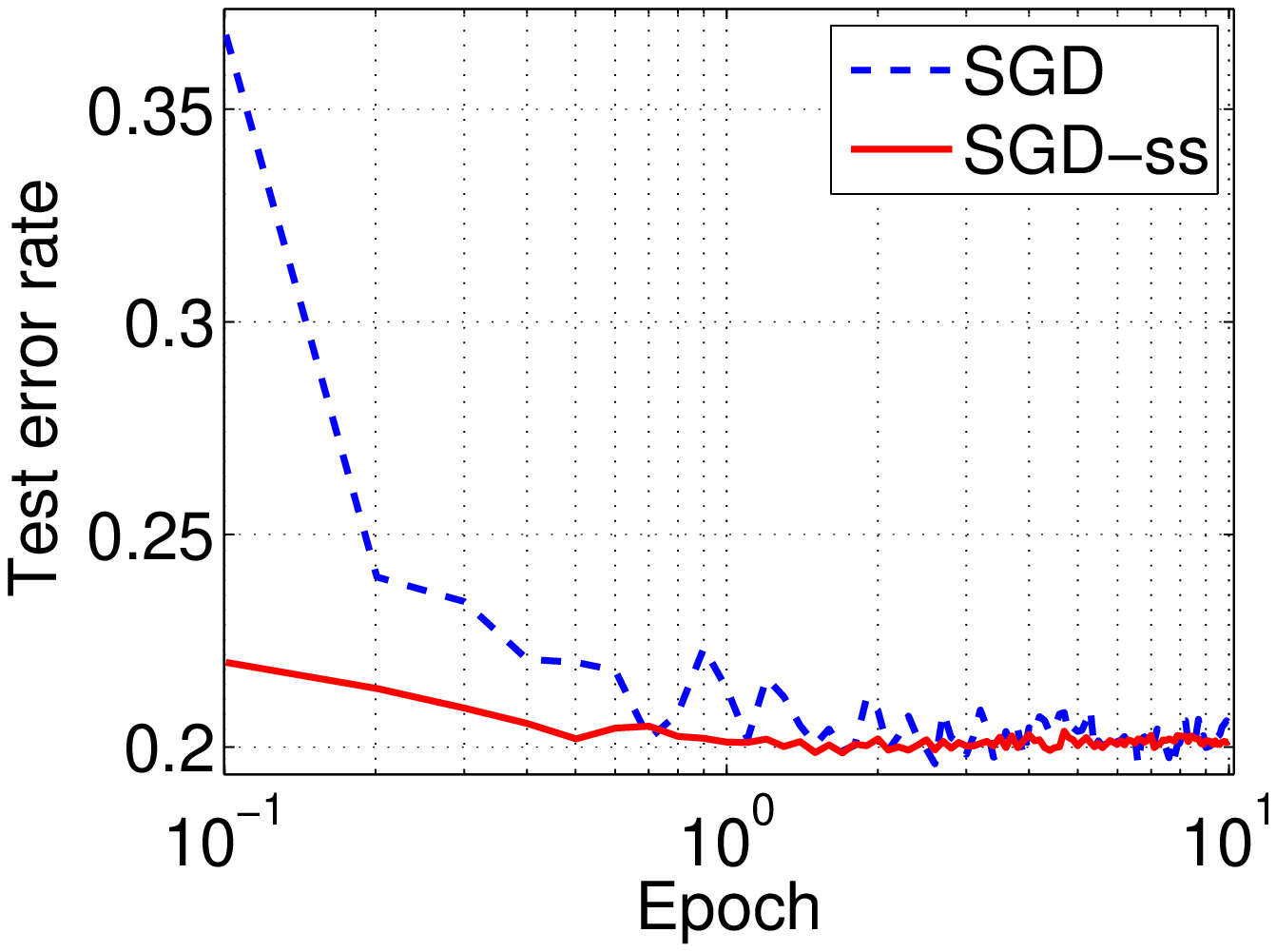}
\includegraphics[width=1.55in]{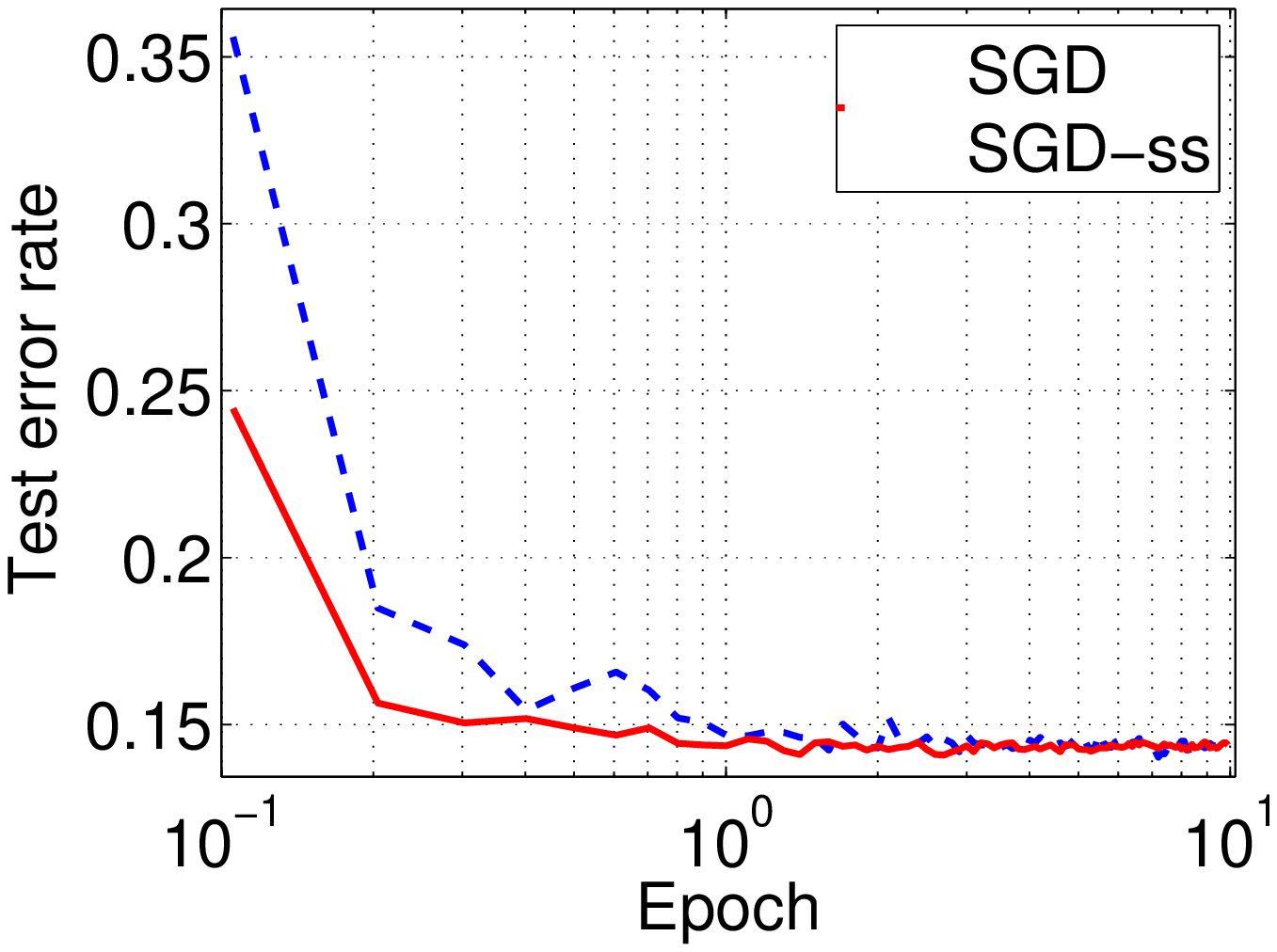}
{\scriptsize \makebox[1.55in]{(e)~letter }~\makebox[1.55in]{(f)~mnist }~\makebox[1.55in]{(g)~pendigits }~\makebox[1.55in]{(h)~ usps}}
\end{center}\vspace{-0.15in}
\begin{center}
\includegraphics[width=1.55in]{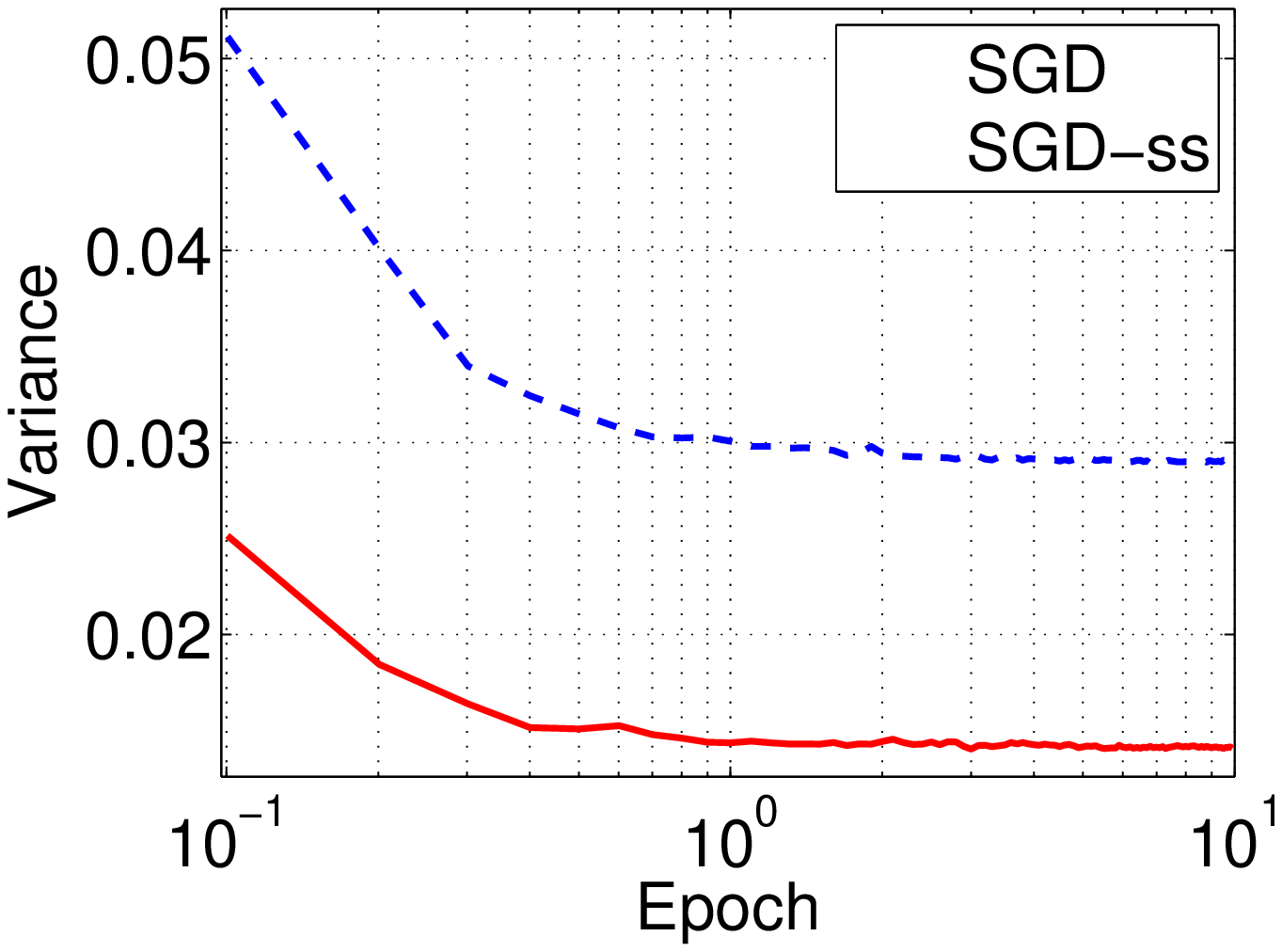}
\includegraphics[width=1.55in]{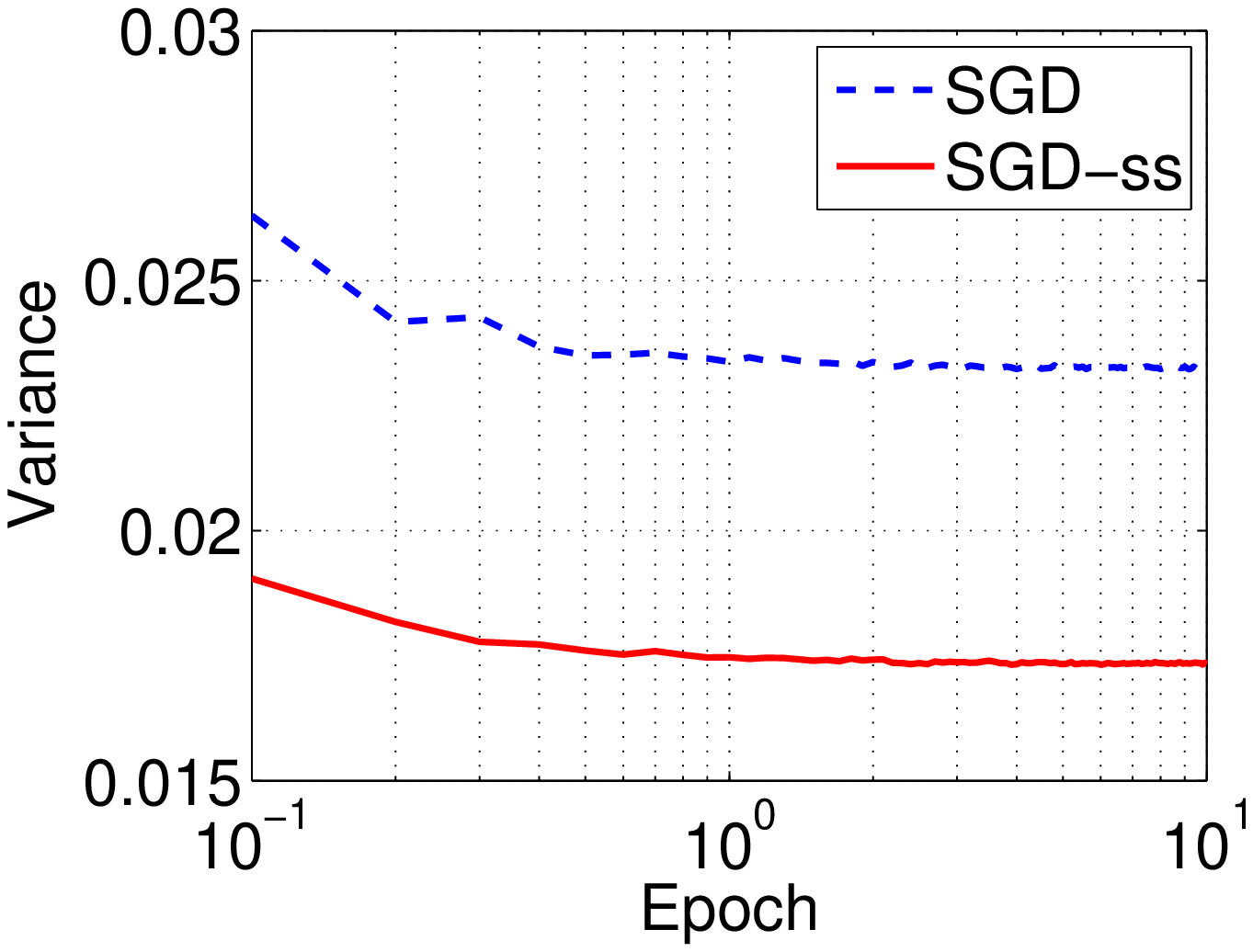}
\includegraphics[width=1.55in]{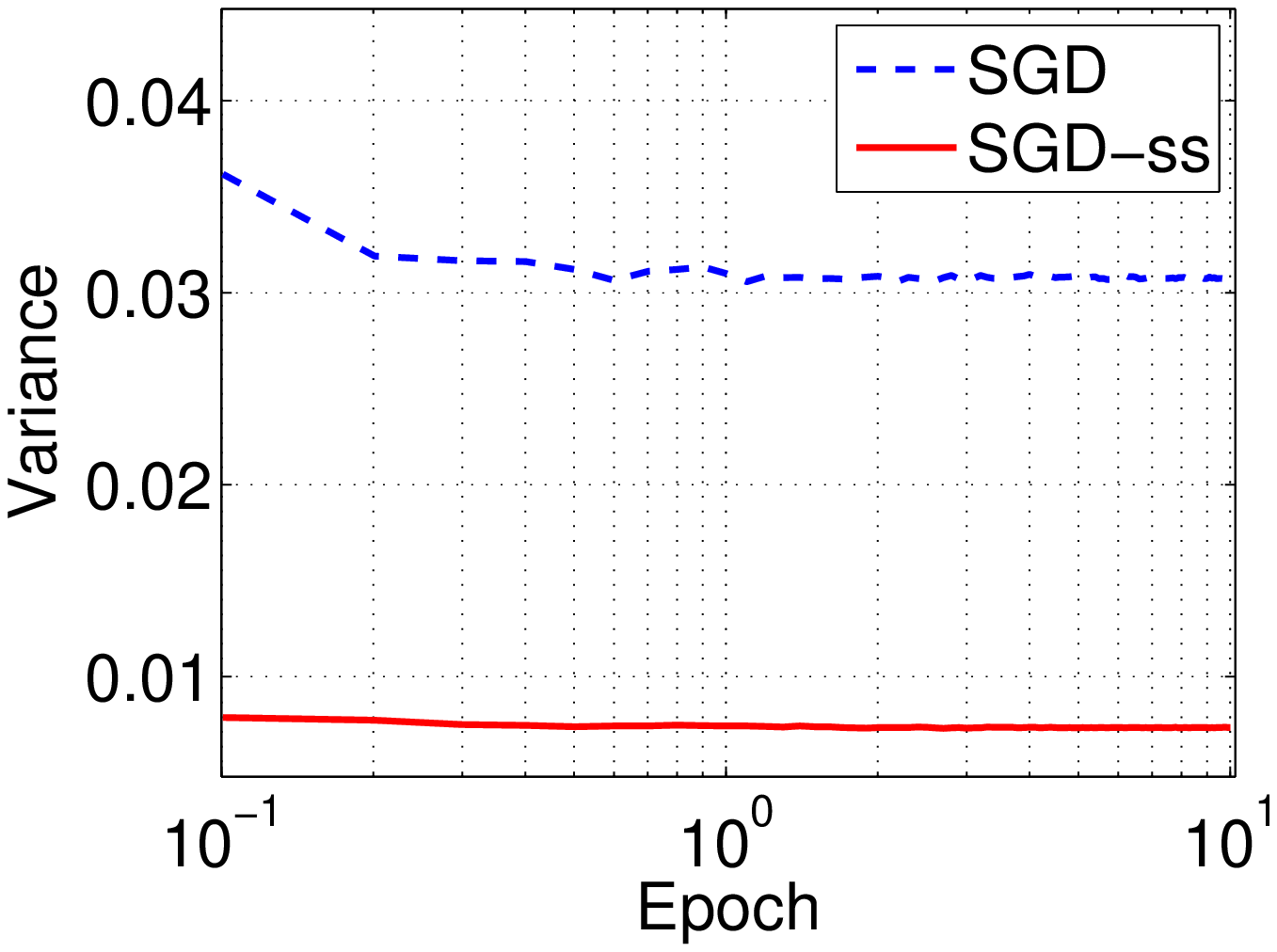}
\includegraphics[width=1.55in]{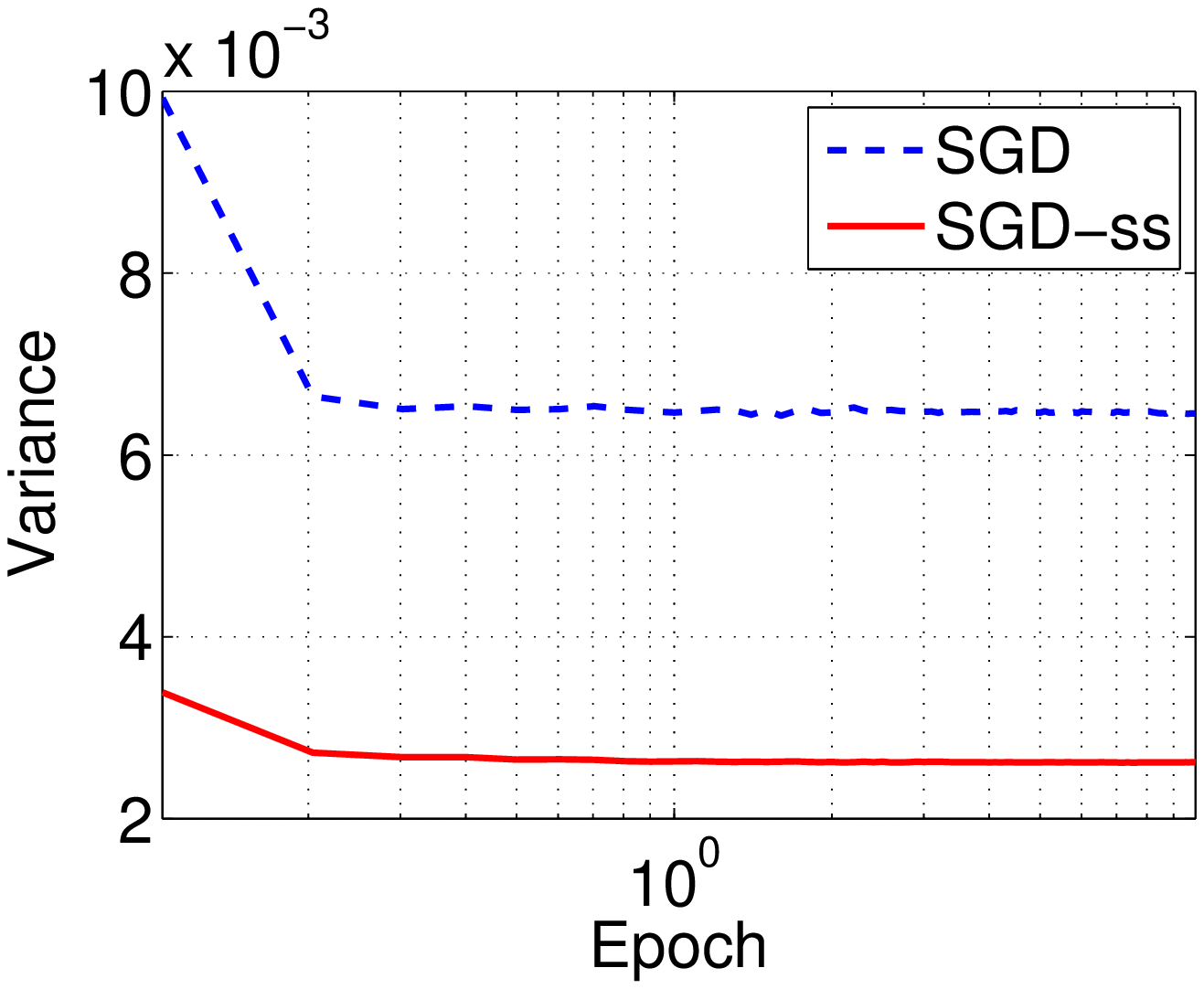}
{\scriptsize \makebox[1.55in]{(i)~letter }~\makebox[1.55in]{(j)~mnist }~\makebox[1.55in]{(k)~pendigits }~\makebox[1.55in]{(l)~ usps}}
\end{center}\vspace{-0.15in}
\caption{More results on multiclass logistic regression (convex) on {\bf letter}, {\bf mnist}, {\bf pendigits}, and {\bf usps}. Epoch for the horizontal axis
is the number of iterations times the minibatch size divided by the training data size. The first row summarized the primal objective value of the algorithms on these four datasets. The second row summarized the test error rates of the algorithms on these four datasets. The final row summarized the variances of the stochastic gradient estimators of the algorithms on these four datasets.}
\label{fig:more}
\end{figure*}
Figure~\ref{fig:more} shows more L2-regularized logistic regression results in terms of primal objective value, test error rate and variance of the stochastic gradients. Overall, SGD-ss  is clearly superior to SGD with uniform sampling, which demonstrates the effectiveness of the proposed sampling strategy.

\section{Conclusion}
\label{sec:conclusion}
This paper studies stratified sampling to reduce the variance for Stochastic Gradient Descent method. We not only provided a dynamic stratified sampling strategy but also provided a fixed stratified sampling strategy. We showed that the convergence rate for the training process can be significantly reduced by the proposed strategy. We have also conducted an extensive set of experiments to compare the proposed method with the traditional uniform sampling strategy. Promising empirical results validated the effectiveness of our technique.
\newpage
\bibliography{reference}
\bibliographystyle{plain}

\end{document}